\documentclass[pmlr]{jmlr}%

\expandafter\let\csname ver@natbib.sty\endcsname\relax
\let\citet\relax
\let\citep\relax

\usepackage{longtable}%

\usepackage{booktabs}
 
\jmlrvolume{255}
\jmlryear{2024}
\jmlrworkshop{ML-DE Workshop at ECAI 2024}

\usepackage[american]{babel}
\usepackage{mathtools} %
\usepackage{amssymb} %
\usepackage[capitalise]{cleveref} %
\usepackage[sortcites,
            sorting=ynt,
            style=authoryear-comp,
            natbib=true, %
            maxcitenames=2, %
            uniquename=false, %
            uniquelist=false, %
            dashed=false, %
            ]{biblatex}
\usepackage{csquotes}
\usepackage{caption}
\usepackage{multirow} %
\newcommand{\multirowvtext}[2]{\multirow{#1}{*}{\begin{tabular}{@{}c@{}}\rotatebox[origin=c]{90}{#2}\end{tabular}}} %

\usepackage{svg} %
\usepackage{wrapfig} %

\usepackage{colortbl}

\renewcommand{\newline}{\mbox{}\\}
\usepackage{pdfcomment}

\newcommand{\s}{\vec{s}}
\newcommand{\x}{\vec{x}}
\newcommand{\y}{\vec{y}}
\newcommand{\W}{\vec{W}}
\newcommand{\U}{\vec{U}}

\newtheorem{restatetheorem}{Theorem}

\title[Hopfield Networks Beyond Synchronous Updates and Forward Euler]{Accelerating Hopfield Network Dynamics:\\Beyond Synchronous Updates and Forward Euler}

\author{\Name{Cédric Goemaere}, \Name{Johannes Deleu}, \Name{Thomas Demeester}
\Email{first.last@ugent.be} \\
\addr IDLab, Department of Information Technology at Ghent University -- imec, Ghent, Belgium\\
}

\editors{Cecília Coelho, Bernd Zimmering, M. Fernanda P. Costa, Luís L. Ferrás, Oliver Niggemann}

\begin{document}

\maketitle

\begin{abstract}
The Hopfield network serves as a fundamental energy-based model in machine learning, capturing memory retrieval dynamics through an ordinary differential equation (ODE).
The model's output, the equilibrium point of the ODE, is traditionally computed via synchronous updates using the forward Euler method.
This paper aims to overcome some of the disadvantages of this approach.
We propose a conceptual shift, 
viewing Hopfield networks 
as instances of Deep Equilibrium Models (DEQs).
The DEQ framework not only allows for the use of specialized solvers, but also leads to new insights on an empirical inference technique that we will refer to as `even-odd splitting'.
Our theoretical analysis of the method uncovers a parallelizable asynchronous update scheme, which should converge roughly twice as fast as the conventional synchronous updates.
Empirical evaluations validate these findings, showcasing the advantages of both the DEQ framework and even-odd splitting in digitally simulating energy minimization in Hopfield networks.
The code is available at \url{https://github.com/cgoemaere/hopdeq}.
\end{abstract}
\begin{keywords}
Even-odd splitting, Hopfield network, Deep Equilibrium Model
\end{keywords}

\section{Introduction}
\label{section_introduction}
In 1982, the Hopfield network was suggested as a model for associative memory retrieval \parencite{hopfield1982}. It restores corrupted memories by solving an ordinary differential equation (ODE) representing the gradient field of a learnable energy function, which holds the true memories at its local minima. In recent years, there has been a renewed interest in Hopfield networks, which has lead to a series of architectural improvements over the original formulation \parencite{krotov2016DAM,demircigil2017MHN,krotov_hopfield2021large_associative_memory_problem,ramsauerMHN2021,krotovHAM2021}. In this paper, we consider two formulations of the Hopfield network: the continuous Hopfield network (CHN) of \textcite{bengio_fischer_CHN_original}, and Hierarchical Associative Memory \parencite[HAM;][]{krotovHAM2021}, which extends the framework of classical continuous Hopfield networks \parencite{hopfield1984CHN} to arbitrary network architectures.

Both during training and inference, Hopfield networks require an internal energy minimization.
Physical compute platforms (i.e., neuromorphic hardware) could solve this optimization problem near-instantaneously and at an extremely low energetic cost \parencite{AnalogCHN}, but unfortunately, such technology is not yet commercially available.
In anticipation of these devices, research on Hopfield networks has turned to the use of traditional digital accelerators, such as GPUs, where solving the ODE is computationally intensive, thereby hindering progress in the field.
Accelerating the digital energy minimization simulations is currently an underexplored research direction.
However, we consider it an essential step in stimulating future research on Hopfield networks in general, especially at larger scales than currently investigated.

Deviating from convention, we argue that compute-intensive ODE-based dynamics may be unnecessary for modelling the behavior of a Hopfield network.
Instead, we propose a conceptual shift, casting Hopfield networks to the framework of Deep Equilibrium Models \parencite[DEQs;][]{baiDeepEquilibriumModels2019DEQ}, which focuses on state dynamics rather than energy.
Through this lens, 
we are able to
theoretically motivate an intuitively appealing idea from \textcite{bengio_even_odd_splitting_original} to help accelerate Hopfield networks and uncover the conditions required for its successful application in practice.

\paragraph{Our contributions:}
\begin{enumerate}
\itemsep=0em
\item We propose a new perspective on Hopfield networks, treating them as DEQs instead of ODEs. This conceptual shift simplifies theoretical analysis and enables the use of specialized solvers that may accelerate convergence.
\item Our analysis reveals that, under specific conditions, CHNs can be interpreted as HAMs, challenging the traditional categorization based solely on energy functions.
\item Revisiting an idea from \textcite{bengio_even_odd_splitting_original}, we uncover its nature as a parallelizable asynchronous update scheme that converges twice as fast, as empirically validated on the MNIST dataset across Hopfield networks of varying sizes.
\end{enumerate}

\noindent
This work expands the scope of our NeurIPS workshop paper \parencite{goemaere2023accelerating} to include both HAMs and CHNs, and offers a much more substantial theoretical analysis and empirical validation, without assuming prior knowledge on Hopfield networks.

\section{Preliminaries}
\label{section_preliminaries}
This section briefly introduces the key concepts underlying the remainder of the paper.
We describe the two types of Hopfield networks considered (CHN and HAM) from the perspective of the current literature. Furthermore, we provide a brief introduction to the DEQ framework and revisit the original idea of \textcite{bengio_even_odd_splitting_original}.
Please note that we introduce recurring symbols in \cref{table_symbols}.

\begin{figure}[t]
\begin{minipage}[c]{.42\linewidth}
\centering
\rowcolors{2}{gray!25}{white}
\begin{small}
\begin{tabular}{m{0.38\linewidth} m{0.6\linewidth}}
    \textbf{Symbol} & \textbf{Description} \\
    \midrule
    $\s\in\mathbb{R}^{N}$ & State vector \\
    $E:\mathbb{R}^{N}\rightarrow\mathbb{R}$ & Global energy function \\
    $\rho:\mathbb{R}^{N}\rightarrow\mathbb{R}^{N}$ & \mbox{Activation function\textsuperscript{\textdagger}} (typically non-linear) \\
    $\mathcal{L}:\mathbb{R}^{N}\rightarrow\mathbb{R}$ & Lagrangian function such that $\frac{\partial\mathcal{L}}{\partial\s} = \rho$ \\
    $\W\in\mathbb{R}^{N\times N}$ & State weight matrix \\
    $\vec{b}\in\mathbb{R}^{N}$ & State bias vector \\
    $\x\in\mathbb{R}^d$ & Input vector\\
    $\U\in\mathbb{R}^{N\times d}$ & Input weight matrix \\
    $\odot$ & Hadamard product\\
    $\langle\cdot\rangle^*$ & Vector $\langle\cdot\rangle$ at equilibrium\\
    $\langle\cdot\rangle^n$ & Vector $\langle\cdot\rangle$ at the $n$-th \mbox{iteration} of its DEQ\\
    \bottomrule
    \rowcolor{white} %
    \multicolumn{2}{p{\linewidth}}{\footnotesize \textdagger: We use a scalar function $\rho$ applied element-wise to the state vector as $\rho(\s)$, corresponding to an additive Lagrangian $\mathcal{L}$ \parencite{krotovHAM2021}.}
\end{tabular}
\end{small}
\captionof{table}{List of symbols} %
\label{table_symbols}
\end{minipage}
\hfill
\begin{minipage}[c]{.51\linewidth}
\hspace{6pt}
\begin{minipage}[c]{.65\linewidth}
    \includeinkscape[width=\linewidth, pretex=\footnotesize]{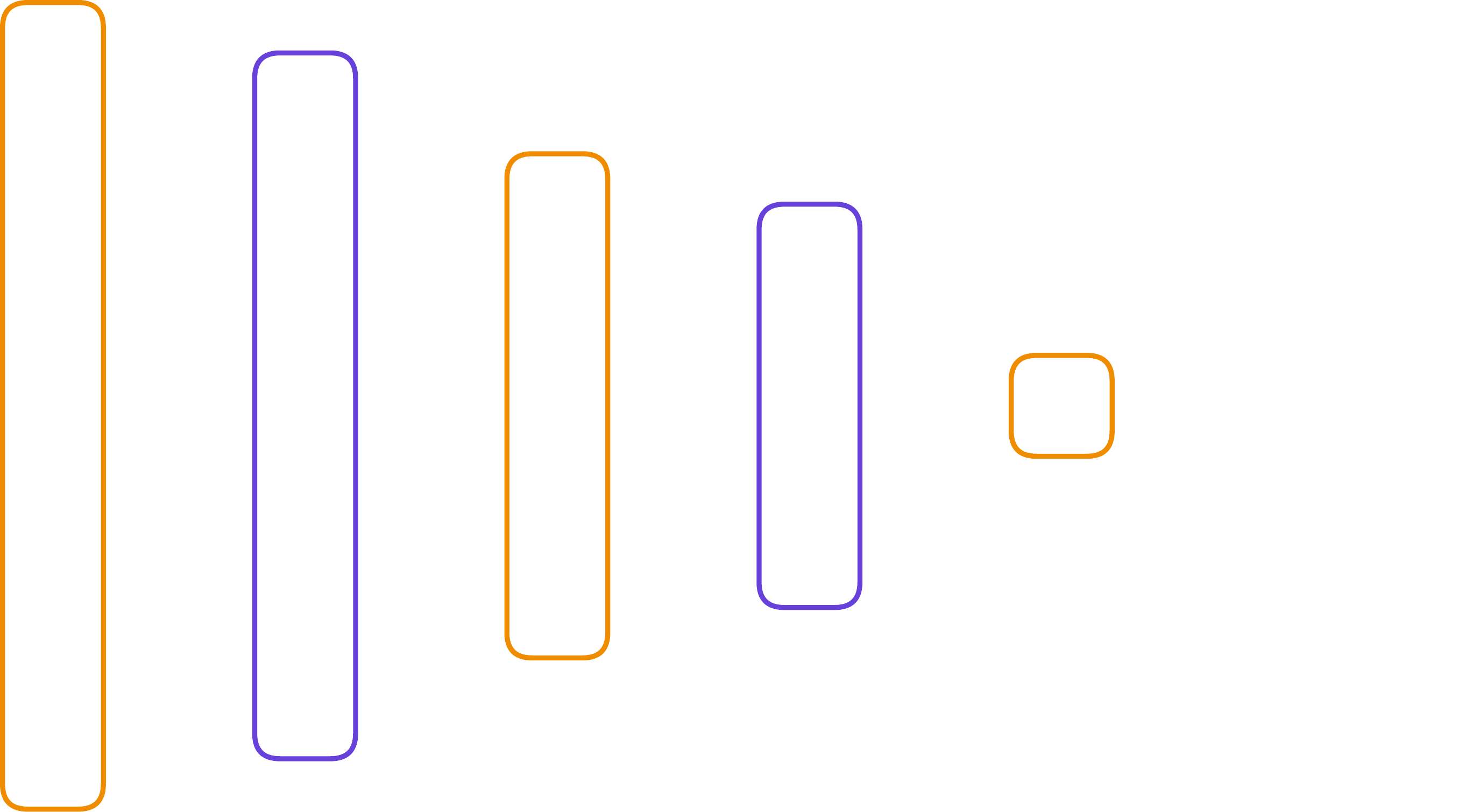}
\end{minipage}
\hspace{-8mm}
\begin{minipage}[c]{.35\linewidth} %
\scalebox{0.9}{%
\(
    \begin{cases}
        \textcolor[HTML]{F08C00}{\s^*_0} &= \x \\
        \textcolor[HTML]{6741D9}{\s^*_1} &= f_{\theta_1}(\textcolor[HTML]{F08C00}{\s^*_0}, \textcolor[HTML]{F08C00}{\s^*_2}) \\
        \textcolor[HTML]{F08C00}{\s^*_2} &= f_{\theta_2}(\textcolor[HTML]{6741D9}{\s^*_1}, \textcolor[HTML]{6741D9}{\s^*_3}) \\
        \textcolor[HTML]{6741D9}{\s^*_3} &= f_{\theta_3}(\textcolor[HTML]{F08C00}{\s^*_2}, \textcolor[HTML]{F08C00}{\s^*_4}) \\
        \textcolor[HTML]{F08C00}{\s^*_4} &= f_{\theta_4}(\textcolor[HTML]{6741D9}{\s^*_3}) = \hat{\y}
    \end{cases}
\)%
}
\end{minipage}\\
\begin{minipage}[c]{.25\linewidth}
\begin{footnotesize}
\begin{equation*}
\s = 
\left\lbrack
\boldmath
\begin{array}{c}
\s_0\\
\s_1\\
\s_2\\
\s_3\\
\s_4
\end{array}
\unboldmath
\right\rbrack
\end{equation*}
\end{footnotesize}
\end{minipage}
\begin{minipage}[c]{.75\linewidth}
\begin{footnotesize}
\begin{equation*}
\W = 
\left\lbrack
\boldmath
\begin{array}{ccccc}
 0   & W_0^T & 0   & 0   & 0   \\
 W_0 & 0   & W_1^T & 0   & 0   \\
 0   & W_1 & 0   & W_2^T & 0   \\
 0   & 0   & W_2 & 0   & W_3^T \\
 0   & 0   & 0   & W_3 & 0  
\end{array}
\unboldmath
\right\rbrack
\end{equation*}
\end{footnotesize}
\end{minipage}
\caption{Diagram of a 5-layer Hopfield network (upper left), with an abstract DEQ formulation of the interlayer dynamics (upper right). State $\s$ and corresponding weight matrix $\W$ (below). Partitioning the layers into \textcolor[HTML]{F08C00}{even} and \textcolor[HTML]{6741D9}{odd} reveals a bipartite structure. Best viewed in color.}
\label{figure_HN_diagram}
\end{minipage}
\end{figure}

\paragraph{Continuous Hopfield network (CHN)}
While originally proposed as a model for associative memory retrieval \parencite{hopfield1984CHN}, \textcite{bengio_fischer_CHN_original} consider the CHN an energy-based model that iteratively updates its hidden neurons to explain signals coming from the sensory neurons, similar to how the brain works.
Based on the Boltzmann machine energy function \parencite{boltzmannmachine1985}, they propose the energy function
\begin{equation}
    E(\s) = \frac{1}{2}||\s||^2  - \frac{1}{2}\rho(\s)^T \W \rho(\s) - \vec{b}^T \rho(\s)
    \label{energy_function_CHN}.
\end{equation}
The zero-diagonal weight matrix $\W$ ($W_{ii}=0$) is often constructed as a symmetric matrix, since any anti-symmetric component is cancelled out in \cref{energy_function_CHN}. By convention, the first $d$ dimensions of the state $\s$ (denoted $\s_0$) constitute the static input $\x$.

In follow-up work, mostly layered instantiations of the CHN have been considered, without intralayer connections \parencite{scellierEquilibriumPropagationBridging2017a,bengio_even_odd_splitting_original,gammellLayerSkippingConnectionsImprove2021,laborieuxImprovingEquilibriumPropagation2023,oconnorINITIALIZEDEQUILIBRIUMPROPAGATION2019,RainPaperNeurIPS},
leading to zero diagonal blocks in $\W$, as illustrated for a 5-layer architecture in \cref{figure_HN_diagram}.

The output of a CHN is $\s^*$, which resides at a minimum of $E$ given $\x$.
The energy $E$ is guaranteed to decrease over time \parencite{scellierEquilibriumPropagationBridging2017a} using the state update rule
\begin{equation}
\frac{d\s}{dt} = -\frac{\partial E}{\partial \s} = -\s + \rho'(\s)\odot\left(\W \rho(\s) + \vec{b}\right)
\label{CHN_state_update_rule}.
\end{equation}
In the literature on Hopfield networks, the equilibrium state $\s^*$ is typically obtained by numerical integration of \cref{CHN_state_update_rule} using the forward Euler method \parencite{bengio_fischer_CHN_original, bengio_even_odd_splitting_original, scellierEquilibriumPropagationBridging2017a, gammellLayerSkippingConnectionsImprove2021}. By contrast, in the field of Neural ODEs \parencite{chen2018neuralODE}, it is customary to use more advanced ODE solvers, and these techniques have been suggested for Hopfield networks as well \parencite{krotovHAM2021}.

\begin{remark}\label{remark_sync_vs_async}
It is common practice to update all states in parallel, known as `synchronous updates'. While fast, this method lacks formal convergence guarantees \parencite{koiran1994dynamics,wang1998_2cycle_convergence}. For guaranteed convergence, one must turn to `asynchronous updates', where the states are sequentially updated, but this can be very slow when applied naively.
\end{remark}

\paragraph{Hierarchical Associative Memory (HAM)}
A HAM \parencite{krotovHAM2021} is the multilayer extension of the classical CHN \parencite{hopfield1984CHN}. The main difference with the CHN from \citeauthor{bengio_fischer_CHN_original} lies in its energy function, which can be defined as
\begin{equation}
    E(\s) = \s^T\rho(\s) - \mathcal{L}(\s)  - \frac{1}{2}\rho(\s)^T \W \rho(\s) -\vec{b}^T\rho(\s)
    \label{energy_function_HAM},
\end{equation}
where the Lagrangian function $\mathcal{L}$ is defined to be the antiderivative of $\rho$ (i.e., $\frac{\partial\mathcal{L}}{\partial\s} = \rho$).

\noindent
By convention, the dynamics of a HAM require the \textit{activations} $\rho(\s)$ to be at equilibrium, rather than the states $\s$. As a result, its state update rule can be simplified to
\begin{equation}
\frac{d\s}{dt} = -\frac{\partial E}{\partial \rho(\s)} = -\s + \W \rho(\s) + \vec{b}
\label{HAM_state_update_rule}.
\end{equation}

\noindent
Notice that the only difference between \cref{CHN_state_update_rule,HAM_state_update_rule} is the missing $\rho'(\s)$-term. All other aspects remain the same, including the implicit input dependence through $\s_0$ and the structure of $\W$ for layered instantiations of the HAM.

\paragraph{Deep Equilibrium Model (DEQ)}
A DEQ \parencite{baiDeepEquilibriumModels2019DEQ} is a recurrent neural network that operates on a static input $\x$.
It returns an equilibrium point $\s^*$, that is defined \emph{implicitly} through the fixed point equation $\s^* = f_\theta(\s^*, \x)$.
The function $f_\theta$ may be any arbitrary computation block, from a simple MLP $\rho(\W\s^* \,{+}\, \vec{b} \,{+}\, \U\x)$ to an elaborate, deep architecture \parencite{baiDeepEquilibriumModels2019DEQ, bai2020mdeq}.

DEQs can be seen as infinite, implicit or adaptive depth models, because the number of iterations -- and hence the depth of the unrolled computational graph -- may be scaled arbitrarily to match the difficulty of the task at hand \parencite{anil2022path_independent_deq}.

Going beyond simple fixed point iteration, DEQs make use of specialized fixed point solvers, such as Anderson acceleration \parencite{anderson1965original_anderson_acceleration, walkerAndersonAccelerationFixedPoint2011} and Broyden's method \parencite{broyden1965method}.
Additionally, their fixed point structure overcomes the need for backpropagation-through-time by offering memory-efficient gradient estimation methods \parencite{baiDeepEquilibriumModels2019DEQ}, based on the implicit function theorem \parencite{krantz2002implicitfunctiontheorem}.

DEQs are part of the broad family of \textit{implicit} models \parencite{implicitmodelswebsite}, that includes Neural Differential Equations \parencite{chen2018neuralODE,kidger2022neuralDE} and differentiable optimization \parencite{amos2019differentiableoptimization}. However, distinguishing between these models can be ambiguous at times, as they often exhibit interchangeable functionality, allowing, for example, Neural ODEs to be represented as DEQs, and vice versa \parencite{kidger2022neuralDE,pal2022ODE_to_infinity_asDEQ}.

\paragraph{Even-odd splitting}
Inspired by block Gibbs sampling in Deep Boltzmann Machines \parencite{salakhutdinov2009deepboltzmannmachines}, \textcite{bengio_even_odd_splitting_original} describe a two-step iterative method to accelerate the energy minimization process of multilayer CHNs. Throughout this paper, we will refer to this technique as `even-odd splitting'.

The intuition is to leverage the problem's bipartite structure (see \cref{figure_HN_diagram}) by splitting it into two smaller, more tractable sub-problems.
In the first step, all odd layers are brought to a local energy minimum, conditioned on the fixed values of the even layers, and in the second step, the roles are reversed.
\citeauthor{bengio_even_odd_splitting_original} argue that iteratively applying these two steps should converge faster than synchronously updating all layers, under the condition that there are no intralayer connections and only connections between successive layers.%
\footnote{Notice how the architecture in \cref{figure_HN_diagram} satisfies both conditions.}

Other than that, the paper does not go into more detail. Even-odd splitting is not its main contribution, and accordingly, no empirical results on the matter are presented.
Despite its intuitive appeal, the method has not gained much traction in the community, possibly because of its ineffectiveness in CHNs in practice (see \cref{section_results}).
Independently, \textcite{RainPaperNeurIPS} introduced a similar technique that shows good performance, but also lacks a theoretical foundation.

\section{Even-odd splitting from a Deep Equilibrium perspective}
\label{section_eo_from_deq_perspective}

Intuitively, even-odd splitting seems like a  promising way to accelerate Hopfield networks, but gaining more theoretical insights is difficult, due to the contemporary energy-based formulation of Hopfield networks. To address this, we propose a conceptual shift towards the DEQ framework, which enables a clearer, step-by-step analysis of the procedure, from the splitting of the states to the local energy optimization.
Our findings reveal that even-odd splitting is not to blame for its ineffectiveness in CHNs. Rather, it is the CHN itself that is inherently inefficient. Contrary to the traditional energy-based distinction, we demonstrate that, under specific conditions, CHNs can be interpreted as HAMs, albeit with a more complicated non-linearity.
Finally, we uncover an elegant mathematical formulation of even-odd splitting in HAMs and discuss its advantages compared to the commonly used synchronous updates.

While the transformation to the DEQ framework primarily serves to streamline our theoretical analysis below, it is worth noting that this reinterpretation also offers practical computational benefits that we exploit in our experiments in \cref{section_results}.
We provide a brief overview of these advantages in \cref{appendix_advantages_deq_framework}.

\subsection{From Hopfield network to DEQ}
\label{section_from_hopfield_to_deq}
Strictly speaking, Hopfield networks are by definition energy-based models, and yet, their dynamics are entirely defined by the ODEs in \cref{CHN_state_update_rule,HAM_state_update_rule}.
In fact, contrary to Neural ODEs \parencite{chen2018neuralODE}, only the equilibrium point matters in a Hopfield network, not the trajectory to get there.
In this particular case, the exact time dynamics are irrelevant, and the problem can be solved much faster by casting the ODE as a DEQ \parencite{pal2022ODE_to_infinity_asDEQ}.
As this has not been previously described for Hopfield networks, we explicitly derive the formulation in \cref{appendix_derivation_deq_HN}, resulting in:
\begin{alignat}{4}
    \textbf{CHN as DEQ:}\qquad
    &\tilde{\s}^* = \rho'(\tilde{\s}^*) \odot(&\tilde{\W}\rho(\tilde{\s}^*) + \tilde{\vec{b}} + \U\rho(\x)&)
    \label{explicit_naive_chn_deq}, \\
    \textbf{HAM as DEQ:}\qquad
    &\tilde{\s}^* = &\tilde{\W}\rho(\tilde{\s}^*) + \tilde{\vec{b}} + \U\rho(\x)&
    \label{explicit_naive_ham_deq}.
\end{alignat}
Here, the tilde on $\tilde{\s}^*$, $\tilde{\W}$, and $\tilde{\vec{b}}$ indicates a change in dimensionality resulting from the separation of the input layer $\s_0 = \x$.
For readability, however, we chose to leave it out in the following sections.

\begin{example}
In a 5-layer Hopfield network like \cref{figure_HN_diagram}, we can easily derive that
\begin{equation*}
\tilde{\W} = 
\left\lbrack
\boldmath
\begin{array}{cccc}
0   & W_1^T & 0   & 0   \\
W_1 & 0   & W_2^T & 0   \\
0   & W_2 & 0   & W_3^T \\
0   & 0   & W_3 & 0  
\end{array}
\unboldmath
\right\rbrack
,\hspace{4pt}
\U = 
\left\lbrack
\boldmath
\begin{array}{ccccc}
W_0  \\
 0   \\
 0   \\
 0   
\end{array}
\unboldmath
\right\rbrack
.
\end{equation*}
\end{example}

\subsection{Splitting the states into even \& odd}
The first step towards understanding even-odd splitting is to introduce a distinction between the even and odd layers, as formalized in \theoremref{theorem_eo_transform} and proven in \cref{appendix_full_derivation_eo}.

\begin{theorem}
\label{theorem_eo_transform}
Even-odd splitting rearranges the layered structure of $\s$ using a permutation matrix $\vec{P}$, such that $\s = [\s_1; \s_2; \s_3; \ldots]$ is converted into $\vec{P}\s = [\s_\textup{\text{even}}; \s_\textup{\text{odd}}]$, with $\s_\textup{\text{even}} = [\s_2,\s_4,\ldots]$ and $\s_\textup{\text{odd}} = [\s_1, \s_3,\ldots]$. Under this change of variables, the fixed point iteration procedures for the CHN and HAM, according to \cref{explicit_naive_chn_deq,explicit_naive_ham_deq}, are transformed into
\begin{align}
\textup{\textbf{CHN iteration:}}\qquad &
\begin{cases}
\s^{n+1}_\textup{\text{even}} &= \rho'(\s^n_\textup{\text{even}}) \odot \vec{W_P^T}\rho(\s^n_\textup{\text{odd}}) + \vec{b}_\textup{\text{even}} \\
\s^{n+1}_\textup{\text{odd}} &= \rho'(\s^n_\textup{\text{odd}}) \odot \vec{W_P} \rho(\s^n_\textup{\text{even}}) + \vec{b}_\textup{\text{odd}} + \U_\textup{\text{odd}}\rho(\x)
\end{cases}
\label{eo_transform_chn_deq_update}%
,\\
\textup{\textbf{HAM iteration:}}\qquad &
\begin{cases}
\s^{n+1}_\textup{\text{even}} &= \vec{W_P^T}\rho(\s^n_\textup{\text{odd}}) + \vec{b}_\textup{\text{even}} \\
\s^{n+1}_\textup{\text{odd}} &= \vec{W_P} \rho(\s^n_\textup{\text{even}}) + \vec{b}_\textup{\text{odd}} + \U_\textup{\text{odd}}\rho(\x)
\end{cases}.
\label{eo_transform_ham_deq_update}
\end{align}
\end{theorem}

\begin{example}
In a 5-layer Hopfield network like \cref{figure_HN_diagram}, we have
\begin{equation*}
\vec{W_P} =
\left\lbrack
\boldmath
\begin{array}{cc}
W_1^T & 0   \\
W_2  & W_3^T 
\end{array}
\unboldmath
\right\rbrack
,\hspace{4pt}
\U_\text{odd} = 
\left\lbrack
\begin{array}{c}
 \vec{W_0}\\
\mathbf{0}
\end{array}
\right\rbrack,
\end{equation*}
with the matrix at block position $(i, j)$ in $\vec{W_P}$ representing the influence of $\s^*_{2j+2}$ on $\s^*_{2i+1}$.
\end{example}

\subsection{Finding the local energy optimum}
\label{section_finding_local_energy_minimum}
The next step is to bring each layer to its local energy minimum, given its neighbors.
When viewed in parallel, the task amounts to computing $\s_\text{even}^*$ given $\s_\text{odd}$, and vice versa.
In this regard, HAMs are exceptionally well suited for even-odd splitting. Given a fixed value of $\s_\text{odd}$ in \cref{eo_transform_ham_deq_update}, the equilibrium value $\s^*_\text{even}$ is retrieved after a single iteration, and vice versa.
By contrast, in CHNs (i.e., \cref{eo_transform_chn_deq_update}), finding this local equilibrium point still requires a computationally expensive fixed point iteration, thereby nullifying any practical benefits (see \cref{section_results}).
Here, the DEQ framework proves particularly effective, enabling us to derive an analytic solution, as detailed and proven in \cref{appendix_correspondence_CHN_HAM}, resulting in \theoremref{theorem_equivalence_CHN_HAM}.
\begin{theorem}\label{theorem_equivalence_CHN_HAM}
    Under relatively mild conditions for $\rho$ and up to an input preprocessing step, a well-behaved CHN can be transformed into a functionally equivalent HAM with effective non-linearity $\rho_\varsigma := \rho\circ \varsigma^{-1}$, where $\varsigma(\s) \,{:=}\, \s \,{\oslash}\, \rho'(\s)$, with ${\oslash}$ representing the Hadamard division.
\end{theorem}
\vspace{-8pt}
\begin{remark}\label{remark_CHN_is_HAM_with_difficult_rho}
In practice, one typically picks a non-linearity $\rho$ that can be quickly computed, such as a sigmoid.
Nevertheless, \theoremref{theorem_equivalence_CHN_HAM} states that a well-behaved CHN internally models a function inversion $\varsigma^{-1}$, which it performs iteratively at inference.
Hence, a CHN can be thought of as a HAM with a non-linearity that is more involved to compute.
\end{remark}

\subsection{DEQ formulation of even-odd splitting in HAMs}
Consider a Hopfield network with a layered architecture, such as the one in \cref{figure_HN_diagram}.
For an odd number $2k{+}1$ of layers, the output layer $\s^*_{2k}$ belongs to $\s^*_\text{even}$, such that $\s_\text{odd}$ serves only as an auxiliary variable, which need not be explicitly modelled, and vice versa for an even number of layers.
In HAMs, this entails a substitution of the expression for $\s^{n+1}_\text{odd}$ in \cref{eo_transform_ham_deq_update} into the one for $\s^{n+2}_\text{even}$, resulting in
\begin{equation}
\s^{n+2}_\text{even} = \vec{W_P^T}\rho\big(\vec{W_P} \rho(\s^n_\text{even}) + \vec{b}_\text{odd} + \U_\text{odd}\rho(\x)\big) + \vec{b}_\text{even}.
\label{even_odd_deq_update_rule}
\end{equation}
From a different perspective, \cref{even_odd_deq_update_rule} may be interpreted as a \textit{single} iteration of the DEQ
\begin{equation}
\s^*_\text{even} = \vec{W_P^T}\rho\big(\vec{W_P} \rho(\s^*_\text{even}) + \vec{b}_\text{odd} + \U_\text{odd}\rho(\x)\big) + \vec{b}_\text{even},
\label{even_odd_deq}
\end{equation}
corresponding to the DEQ formulation of even-odd splitting in HAMs and certain CHNs (by \theoremref{theorem_equivalence_CHN_HAM}).

\subsection{Advantages of even-odd splitting in Hopfield networks}
\label{section_advantages_eo}
In Hopfield networks, even-odd splitting comes with two notable advantages over the traditional synchronous updates: convergence is reached faster and is always guaranteed.

\paragraph{Faster convergence}
By advancing two time steps (i.e., $\s_\text{even}^n \!\rightarrow \s_\text{even}^{n+2}$ in \cref{even_odd_deq_update_rule}) in a single iteration, even-odd splitting should converge twice as fast as fully synchronous updates of \cref{explicit_naive_ham_deq}, as given by \cref{eo_transform_ham_deq_update}.
Crucially, one fixed point iteration of \cref{even_odd_deq} still requires \emph{exactly} as many computations as an iteration of \cref{explicit_naive_ham_deq}
when not accounting for multiplications by zero (i.e., assuming an optimal block-sparse matrix multiplication).
To verify this, observe that a single iteration always corresponds to all states being updated exactly once (in \cref{even_odd_deq_update_rule}: first, $\s_\text{odd}$ at step $n\!+\!1$ and then $\s_\text{even}$ at step $n\!+\!2$).

\paragraph{Guaranteed convergence}
As mentioned in \remarkref{remark_sync_vs_async}, asynchronous update schemes have convergence guarantees for Hopfield networks, unlike synchronous updates. Interestingly, the order in which the states are asynchronously updated, traditionally chosen to be random, does not influence this property \parencite{koiran1994dynamics}.%
\footnote{The order does, however, influence the convergence \textit{speed}.}
Through a clever grouping of the states into even and odd layers, even-odd splitting essentially establishes an ordering of asynchronous state updates that achieves maximal parallelism in layered architectures.

In \cref{appendix_redundancy_of_synchronous_updates}, we provide some intuition into the problems that may occur when using synchronous updates and highlight how even-odd splitting naturally overcomes them.

\section{Experimental results}
\label{section_results}

\setlength{\columnsep}{17pt}
\begin{table}[!htb]
\centering
\resizebox{0.85\width}{!}{%
\begin{tabular}{ccccc}
    & Model & \#iters to conv. & Speedup & Test acc.~(\%) \\
    \midrule
    \multirowvtext{5}{3 layers}
    & CHN (10 epochs) & 75.5 ($\pm$2.1) & 0.5x & 97.9 ($\pm$0.1)\\
    & CHN (3 epochs) & 39.1 ($\pm$3.8) & 1x & 97.0 ($\pm$0.2)\\
    & CHN-DEQ\textsuperscript{$\alpha$} & 20.6 ($\pm$0.2) & \textbf{1.9x} & 97.2 ($\pm$0.3)\\
    & CHN-EO\textsuperscript{$\beta$} & 16.8 ($\pm$0.5) & 0.1x & 97.1 ($\pm$0.1)\\
    & CHN-EO-DEQ\textsuperscript{$\beta$} & \textbf{16.2} ($\pm$1.0) & 0.1x & 97.1 ($\pm$0.2)\\
    \midrule
    \multirowvtext{4}{3 layers}
    & HAM & 11.9 ($\pm$0.4) & 1x & 97.9 ($\pm$0.0)\\
    & HAM-DEQ & 9.9 ($\pm$0.3) & 1.2x & 97.9 ($\pm$0.1)\\
    & HAM-EO & 8.0 ($\pm$0.2) & 1.5x & 97.9 ($\pm$0.1)\\
    & HAM-EO-DEQ & \textbf{6.6} ($\pm$0.2) & \textbf{1.8x} & 97.9 ($\pm$0.1)\\
    \midrule
    \multirowvtext{4}{5 layers}
    & HAM & 36.0 ($\pm$1.8) & 1x & 97.1 ($\pm$0.1)\\
    & HAM-DEQ & 33.0 ($\pm$0.6) & 1.1x & 97.1 ($\pm$0.2)\\
    & HAM-EO & 18.3 ($\pm$0.5) & 2.0x & 97.1 ($\pm$0.1)\\
    & HAM-EO-DEQ & \textbf{17.7} ($\pm$0.3) & \textbf{2.0x} & 97.1 ($\pm$0.1)\\
    \midrule
    \multirowvtext{4}{7 layers}
    & HAM & 67.1 ($\pm$2.9) & 1x & 95.6 ($\pm$0.2)\\
    & HAM-DEQ & 56.0 ($\pm$1.4) & 1.2x & 95.6 ($\pm$0.1)\\
    & HAM-EO & 32.2 ($\pm$0.8) & 2.1x & 95.5 ($\pm$0.2)\\
    & HAM-EO-DEQ & \textbf{31.0} ($\pm$1.0) & \textbf{2.2x} & 95.5 ($\pm$0.2)\\
    \bottomrule
    \multicolumn{5}{l}{\footnotesize $\alpha$: trained for only 4 epochs; consistently became unstable during 5\textsuperscript{th} epoch}\\
    \multicolumn{5}{l}{\footnotesize $\beta$: 1 iteration of even-odd splitting in CHNs comprises 2x10 local iterations}
\end{tabular}%
}
\caption{Impact of DEQ solver (`DEQ') and even-odd splitting (`EO') on the mean number of iterations until convergence and MNIST test accuracy. Speedup refers to the reduction in state updates needed to reach convergence w.r.t.\ the base model (including within-iteration updates). Results aggregated across 5 runs.}
\label{table_results}
\vspace{-13pt}
\end{table}

\subsection{Setup}

\noindent
To assess the impact of even-odd splitting and DEQ solvers in Hopfield networks,
we performed an ablation study on several CHNs and HAMs of different depths. The models were trained on the MNIST dataset \parencite{lecun1998mnist,cohen2017emnist} for 10 epochs, with the relevant techniques active both \mbox{during} training and testing.

We evaluated convergence speed and performance on the test set and list the results in \cref{table_results}.
Rather than relying on wall time or FLOPS, which depend on the exact implementation and hardware, we quantify convergence speed by the number of iterations required to reach convergence\footnote{In our experiments, `convergence' denotes a relative residual $\frac{||\s^{n+1} - \s^n||_2}{||\s^{n+1}||_2} < 10^{-4}$.}.
Recall, however, that one iter\-ation of even-odd splitting invol\-ves finding two local equilibrium points ($\s^*_\text{odd}$ \& $\s^*_\text{even}$).
In HAMs, each point is retrieved in a single step, but in CHNs, it may require several iterations (here: 10).\\

\noindent
Further details are provided in \cref{appendix_experimental_setup}.

\subsection{Discussion}

Our experiments reveal that \textbf{combining the DEQ solver with even-odd splitting significantly accelerates convergence}, with both methods effective individually and most impactful when used together (see \cref{table_results}).
This aligns with our expectations from \cref{section_advantages_eo}, where we predicted that even-odd splitting would halve the iteration count.

\textbf{In CHNs}, we observed a trade-off between speed and test accuracy.
The vanilla CHN gradually demands more iterations as training progresses, a common phenomenon in DEQs \parencite{baiStabilizingEquilibriumModels2021}. In the other configurations, though, it was far less prominent.

To ensure a fair comparison, we limited the vanilla CHN’s training to 3 epochs.
Even then, the DEQ solver still shows a substantial speedup, as it constitutes a more powerful method to solve the CHN's function inversion at inference (see \remarkref{remark_CHN_is_HAM_with_difficult_rho}).
Conversely, even-odd splitting does not offer a practical speedup. The added cost of resolving local equilibria in each iteration significantly outweighs the reduction in total iterations.

\textbf{In HAMs}, both the DEQ solver and even-odd splitting enhance convergence without trade-offs in test accuracy.
The simpler model dynamics make these methods particularly effective, delivering optimal performance when combined.
In contrast to its substantial role in CHNs, the DEQ solver’s more limited advantage in HAMs indicates that, here, its value primarily lies in stabilizing initial conditions, where early dynamics differ from those in the stable regime.
For a visual comparison of the state dynamics in the different models, we refer the reader to \cref{figure_state_dynamics_3layers_CHN_HAM,figure_state_dynamics_5-7layers} in \cref{appendix_visual_state_dynamics}.

\section{Limitations}
\label{section_limitations}
Our experiments focused on the intrinsic effects of DEQ solvers and even-odd splitting on the \textit{convergence speed} of Hopfield networks, which may not directly translate to reduced \textit{compute times}. For small architectures simulated on modern processors, the added solver complexity or the sequential nature of even-odd splitting could offset the convergence gains. In our tests, even-odd splitting showed notable computational benefits, whereas DEQ solvers were slightly less efficient. Although these methods may offer greater advantages with larger models, our conclusions on scalability are limited given this study's proof-of-concept scope.

State-of-the-art performance was not our objective; instead, we opted for minimal hyperparameter tuning, only enough to ensure model stability. The impact of design choices (e.g., initialization, non-linearity, optimizer, number of layers) in Hopfield networks remains underexplored, leaving room for further improvement.

One key aspect is the choice of Lagrangian $\mathcal{L}$, which defines the family of HAM models.
In particular, we did not investigate the HAM extension of the Modern Hopfield Network \parencite{ramsauerMHN2021,krotov_hopfield2021large_associative_memory_problem}, which could be an interesting direction for future research.

Lastly, while we focused on Hopfield networks, even-odd splitting may also benefit other architectures with bipartite structures.
Similarly, DEQ solvers -- often designed as root finders -- could be useful in accelerating other energy-based models, provided that the gradient can be formulated easily.

\section{Conclusion}

The goal of this paper was to accelerate the digital simulation of energy minimization in Hopfield networks.
To that end, we proposed a conceptual shift, away from the traditional energy-based view, toward the DEQ framework, which allows for simpler analysis and offers several computational advantages, such as specialized solvers and lower memory complexity.
This perspective enabled us to derive theoretical underpinnings for even-odd splitting, an intuitive idea from \textcite{bengio_even_odd_splitting_original}, and uncovered a correspondence between two commonly used types of Hopfield networks, the CHN and HAM. Our analysis revealed that the CHN essentially performs a function inversion at inference, which may not be computationally optimal.
The experimental results demonstrate the effectiveness of both the DEQ framework (with its specialized solvers) and even-odd splitting, especially when combined. Specifically, we observed that these techniques reduced the required amount of iterations to reach convergence, without compromising on test accuracy.

In light of these findings, we advocate for the use of the DEQ framework as a basis for both theoretical analysis and practical implementation of Hopfield networks, given its concise, equilibrium-focused notation and its computational efficiency.
For researchers looking to speed up their Hopfield networks without sacrificing performance, we recommend the use of DEQ solvers and even-odd splitting, alongside traditional solutions like optimized code and high-performance hardware.
Additionally, we encourage researchers to focus on HAMs instead of CHNs, as the primary difference lies in the choice of non-linearity, with HAMs being more computationally efficient.
The tools and methodologies presented in this work aim to facilitate the practical scaling-up of Hopfield networks, which we hope will stimulate further research into this growing field.

\acks{We are grateful to the anonymous reviewers for their valuable feedback and suggestions, which helped improve the presentation and clarity of this paper.\\
This research was funded by the Research Foundation -- Flanders (FWO-Vlaanderen) through CG's PhD Fellowship fundamental research with grant number 11PR824N, and the research project fundamental research G0C2723N.}

\newrefcontext[sorting=nyt] %
\printbibliography

\newpage

\appendix

\newrefcontext[sorting=ynt] %

\section{Advantages of viewing Hopfield networks as DEQs}
\label{appendix_advantages_deq_framework}
Below, we describe the advantages that come with our transition from the energy-based view of Hopfield networks to the DEQ framework of \cref{section_from_hopfield_to_deq}.

First, we would like to highlight that the DEQ formulation fully encompasses the Hopfield network's original ODE formulation.
In fact, solving the DEQs in \cref{explicit_naive_chn_deq,explicit_naive_ham_deq} with a damped Picard iteration is mathematically equivalent to solving the model's ODE with the forward Euler method, where the damping factor corresponds to the time step size.
Using specialized DEQ solvers allows for even faster convergence\footnote{The DEQ solver cannot guarantee energy minimization and may move towards spurious extrema. In this work, however, we will assume that it always finds the true energy minimum.}, as shown in \cref{section_results}.

Additionally, while Hopfield networks are typically trained using backpropagation-through-time, DEQs offer more memory-efficient methods, such as recurrent backpropagation \parencite{almeida1987recurrentbackprop,pineda1987recurrentbackprop}.
While similar algorithms have been suggested for the energy-based setting \parencite{scellierEquilibriumPropagationBridging2017a}, they aim to approximate these exact methods and are often sensitive to the exact choice of hyperparameters \parencite{RainPaperNeurIPS}.

Furthermore, unlike for Hopfield networks, the stability of DEQs is a widely studied area, that includes regularization terms and even parametrizations that are provably stable \parencite{ghaouiImplicitDeepLearning2020, baiDeepEquilibriumModels2019DEQ, baiStabilizingEquilibriumModels2021, revayLipschitzBoundedEquilibrium2020, winstonMonotoneOperatorEquilibrium2021}.

Moving past computational advantages, we argue that the DEQ framework is a more natural way of reasoning about the dynamics of Hopfield networks, as it offers a comprehensible, concise formulation that operates directly at the equilibrium level. Compared to the energy-based setting, the DEQ framework makes it significantly easier to study the characteristics of techniques like even-odd splitting, that act on the states at equilibrium.

While the close relationship between DEQs and Hopfield networks has been noticed before \parencite{krotovHAM2021, otaIMixerHierarchicalHopfield2023, laborieuxImprovingEquilibriumPropagation2023}, remarkably, none of the many advantages described here are exploited in these works.

\section{Derivation of DEQ formulation of CHN \& HAM}
\label{appendix_derivation_deq_HN}
Given the expression for the state update rule $\frac{d\s}{dt}$ in a Hopfield network, we may implicitly define $\s^*$ as the equilibrium state for which $\frac{d\s}{dt}(\s^*) = \mathbf{0}$, thereby yielding the desired DEQ formulation.
We work out the details below for the CHN and HAM.

\paragraph{DEQ formulation of CHN}\newline
For CHNs, $\frac{d\s}{dt}$ is given by \cref{CHN_state_update_rule}. Setting this to zero, we find the following DEQ:
\begin{equation}
    \s^* = \rho'(\s^*) \odot (\W \rho(\s^*) + \vec{b})
    \label{implicit_naive_chn_deq}.
\end{equation}
At first glance, this DEQ seems to be independent of the input and therefore always converge to the same $\s^*$. However, recall that the input $\x$ is \emph{implicitly} applied through the first $d$ states.

When the equilibrium state $\s^*$ is split up into the input $\x$ and hidden state $\tilde{\s}^*$, i.e., $\s^* = [\x;\tilde{\s}^*]$, we can reformulate \cref{implicit_naive_chn_deq} with an explicit input dependence as
\begin{equation}
    \tilde{\s}^* = \rho'(\tilde{\s}^*) \odot(\tilde{\W}\rho(\tilde{\s}^*) + \tilde{\vec{b}} + \U\rho(\x)),
\end{equation}
where the tilde on $\tilde{\s}^*$, $\tilde{\W}$ and $\tilde{\vec{b}}$ indicates a change in dimensions due to the slicing operation.

\begin{example}
In a 5-layer CHN like \cref{figure_HN_diagram}, we can easily derive that
\begin{equation*}
\tilde{\W} = 
\left\lbrack
\boldmath
\begin{array}{cccc}
0   & W_1^T & 0   & 0   \\
W_1 & 0   & W_2^T & 0   \\
0   & W_2 & 0   & W_3^T \\
0   & 0   & W_3 & 0  
\end{array}
\unboldmath
\right\rbrack
,\hspace{4pt}
\U = 
\left\lbrack
\boldmath
\begin{array}{ccccc}
W_0  \\
 0   \\
 0   \\
 0   
\end{array}
\unboldmath
\right\rbrack
.
\end{equation*}
\end{example}

\paragraph{DEQ formulation of HAM}\newline
For HAMs, the derivation is entirely identical as for CHNs. Starting from \cref{HAM_state_update_rule}, we find
\begin{equation}
    \tilde{\s}^* = \tilde{\W}\rho(\tilde{\s}^*) + \tilde{\vec{b}} + \U\rho(\x).
\end{equation}
For readability, we leave out the tilde in the rest of the paper.

\section{Splitting the states into even \& odd: full derivation}
\label{appendix_full_derivation_eo}
Below, we provide a detailed derivation of \cref{eo_transform_chn_deq_update,eo_transform_ham_deq_update}, serving as a proof for \theoremref{theorem_eo_transform}.
We begin with the more concise derivation of even-odd splitting in HAMs. For the CHN, the process is entirely analogous, therefore, we provide only a rough sketch of the derivation.

\subsection{Derivation of even-odd splitting in HAMs}
First, we multiply both sides of \cref{explicit_naive_ham_deq} with a general permutation matrix $\vec{P}$ and find:
\begin{align}
    \vec{P}\s^* &= \vec{P}\W \rho(\s^*) + \vec{P}\vec{b}+\vec{P}\U\rho(\x) \nonumber\\
                &= \vec{P}\W\vec{P}^T \vec{P}\rho(\s^*) + \vec{P}\vec{b}+\vec{P}\U\rho(\x) \nonumber\\
                &= \vec{P}\W\vec{P}^T \rho(\vec{P}\s^*) + \vec{P}\vec{b}+\vec{P}\U\rho(\x)
    \label{P_transformed_explicit_naive_ham_deq},
\end{align}
where $\vec{P}$ can be brought inside $\rho$, as it applies the same element-wise non-linearity over all states.%
\footnote{For a more general $\rho$, one should use a permuted version $\rho_{_{\vec{P}}}$.}
In even-odd splitting, $\vec{P}$ transforms $\s^* = [\s_1^*; \s_2^*; \s_3^*; \ldots]$ into $\vec{P}\s^* = [\s^*_\text{even}; \s^*_\text{odd}]$, with $\s^*_\text{even} = [\s^*_2,\s^*_4,\ldots]$ and $\s^*_\text{odd} = [\s^*_1, \s^*_3,\ldots]$. Using this specific $\vec{P}$, we find
\begin{align*}
\vec{P} \W \vec{P^T} &= 
\left\lbrack
\begin{array}{cc}
 \mathbf{0}   & \vec{W_P^T} \\
\vec{W_P} & \mathbf{0}
\end{array}
\right\rbrack,
&
\vec{P} \s^* &= 
\left\lbrack
\begin{array}{c}
\s^*_\text{even}\\
\s^*_\text{odd}
\end{array}
\right\rbrack,
\\ %
\vec{P} \vec{b} &= 
\left\lbrack
\begin{array}{c}
\vec{b}_\text{even}\\
\vec{b}_\text{odd}
\end{array}
\right\rbrack,
&
\vec{P} \U &= 
\left\lbrack
\begin{array}{c}
\mathbf{0} \\
\U_\text{odd}
\end{array}
\right\rbrack.
\end{align*}
\begin{example}
In a 5-layer HAM like \cref{figure_HN_diagram}, we have
\begin{equation*}
\vec{W_P} =
\left\lbrack
\boldmath
\begin{array}{cc}
W_1^T & 0   \\
W_2  & W_3^T 
\end{array}
\unboldmath
\right\rbrack
,\hspace{4pt}
\U_\text{odd} = 
\left\lbrack
\begin{array}{c}
 \vec{W_0}\\
\mathbf{0}
\end{array}
\right\rbrack,
\end{equation*}
with the matrix at block position $(i, j)$ in $\vec{W_P}$ representing the influence of $\s^*_{2j+2}$ on $\s^*_{2i+1}$. The locations of the zero matrices in $\vec{W_P}$ and $\U_\textup{\text{odd}}$ correspond to skip connections, which are technically allowed, as long as they are between even and odd layers.\footnote{The second condition of \citeauthor{bengio_even_odd_splitting_original} was phrased too restrictively.}
\end{example}
With the proper substitutions in \cref{P_transformed_explicit_naive_ham_deq}, we find
\begin{equation}
\begin{cases}
\s^*_\text{even} &= \vec{W_P^T}\rho(\s^*_\text{odd}) + \vec{b}_\text{even} \\
\s^*_\text{odd} &= \vec{W_P} \rho(\s^*_\text{even}) + \vec{b}_\text{odd} + \U_\text{odd}\rho(\x)
\end{cases}
\label{explicit_ham_eo_deq}
\end{equation}
Therefore, the synchronous state update rule corresponding to the fixed point iteration procedure for the HAM, according to \cref{explicit_naive_ham_deq},
\begin{equation*}
\s^{n+1} = \vec{W}\rho(\s^n) + \vec{b} + \U\rho(\x),
\end{equation*}
in which the state superscript $n$ denotes the iteration index,
can be written as

\begin{equation*}
\left\lbrack
\begin{array}{c}
\s^{n+1}_\text{even}\\
\s^{n+1}_\text{odd}
\end{array}
\right\rbrack
=
\left\lbrack
\begin{array}{cc}
 \mathbf{0}   & \vec{W_P^T} \\
\vec{W_P} & \mathbf{0}
\end{array}
\right\rbrack
\rho
\left(
\left\lbrack
\begin{array}{c}
\s^{n}_\text{even}\\
\s^{n}_\text{odd}
\end{array}
\right\rbrack
\right)
+
\left\lbrack
\begin{array}{c}
\vec{b}_\text{even}\\
\vec{b}_\text{odd}
\end{array}
\right\rbrack
+
\left\lbrack
\begin{array}{c}
\mathbf{0} \\
\U_\text{odd}
\end{array}
\right\rbrack
\rho(\x),
\end{equation*}
or simplified,
\begin{equation*}
\begin{cases}
\s^{n+1}_\text{even} &= \vec{W_P^T}\rho(\s^n_\text{odd}) + \vec{b}_\text{even} \\
\s^{n+1}_\text{odd} &= \vec{W_P} \rho(\s^n_\text{even}) + \vec{b}_\text{odd} + \U_\text{odd}\rho(\x)
\end{cases}.
\end{equation*}

\subsection{Derivation of even-odd splitting in CHNs}
The derivation of even-odd splitting in CHNs is entirely analogous to the one for HAMs, with the subtle difference of the added $\rho'$-term. This mainly poses a challenge in finding the equivalent for \cref{P_transformed_explicit_naive_ham_deq}, as it requires Proposition \ref{prop_permutated_Hadamard}.

\begin{proposition}\label{prop_permutated_Hadamard}
    For a permutation matrix $\vec{P}$ and vectors $\vec{a}, \vec{b} \in \mathbb{R}^{N}$: $\vec{P} (\vec{a} \odot  \vec{b}) =  \vec{Pa} \odot  \vec{Pb}$
\end{proposition}
Using Proposition \ref{prop_permutated_Hadamard}, we may start from
\begin{align*}
    \vec{P}\s^* &= \vec{P}\rho'(\s^*) \odot (\vec{P}\W \rho(\s^*) + \vec{P}\vec{b}+\vec{P}\U\rho(\x)),
\end{align*}
from which we can proceed, entirely analogously to the case of the HAM, to eventually find
\begin{equation*}
\begin{cases}
\s^{n+1}_\textup{\text{even}} &= \rho'(\s^n_\textup{\text{even}}) \odot \vec{W_P^T}\rho(\s^n_\textup{\text{odd}}) + \vec{b}_\textup{\text{even}} \\
\s^{n+1}_\textup{\text{odd}} &= \rho'(\s^n_\textup{\text{odd}}) \odot \vec{W_P} \rho(\s^n_\textup{\text{even}}) + \vec{b}_\textup{\text{odd}} + \U_\textup{\text{odd}}\rho(\x)
\end{cases}.
\end{equation*}

\section{Correspondence between CHN and HAM}
\label{appendix_correspondence_CHN_HAM}
As outlined in \cref{section_preliminaries}, the distinction between a CHN and a HAM is conventionally made on the basis of the energy functions.
However, via the DEQ formulations of \cref{section_from_hopfield_to_deq}, we can demonstrate a correspondence between these two models, as formalized in \theoremref{theorem_equivalence_CHN_HAM}. This appendix gradually builds up towards the theorem's proof.

\begin{definition}
    In a well-behaved CHN, $\s^*$ is always unique and not identically zero.
\end{definition}
\begin{lemma}\label{lemma_CHN_single_rhod_nonzero}
    A CHN with an element-wise $\rho$ can only be well-behaved if $\rho'(0) \,{\neq}\, 0$.
\end{lemma}
\vspace{-8pt}
\begin{proof}
    (Contraposition) If $\rho'(0) \,{=}\, 0$, then $\s^* \,{=}\, \mathbf{0}$ is always an equilibrium state of the CHN. Thus, $\s^*$ is either identically zero or not unique, and the CHN is not well-behaved.
\end{proof}
\vspace{-8pt}
\begin{lemma}\label{lemma_CHN_all_rhod_nonzero}
    In a well-behaved CHN with an element-wise $\rho$,  $\rho'(\s^*)$ contains no zeros.
\end{lemma}
\vspace{-8pt}
\begin{proof}
    \parencite[from][]{bengio_fischer_CHN_original} Consider the $i$-th state $s_i^*$ and assume $\rho'(s_i^*) \,{=}\, 0$. Then, in the $i$-th equation of \cref{explicit_naive_chn_deq}, the right-hand side equals 0, reducing the whole to $s_i^* \,{=}\, 0$. However, by \lemmaref{lemma_CHN_single_rhod_nonzero}, we know that in a well-behaved CHN $\rho'(s_i^*) \,{=}\, 0 \implies s_i^* \,{\neq}\, 0$. This is a contradiction.
\end{proof}
\vspace{-8pt}
\begin{restatetheorem}[restated]
    Under relatively mild conditions for $\rho$ and up to an input preprocessing step, a well-behaved CHN can be transformed into a functionally equivalent HAM with effective non-linearity $\rho_\varsigma := \rho\circ \varsigma^{-1}$, where $\varsigma(\s) \,{:=}\, \s \,{\oslash}\, \rho'(\s)$, with ${\oslash}$ representing the Hadamard division.
\end{restatetheorem}
\begin{proof}
For a well-behaved CHN, \lemmaref{lemma_CHN_all_rhod_nonzero} allows us to rewrite the DEQ of \cref{explicit_naive_chn_deq} as
\begin{equation*}
    \varsigma(\s^*) = \W\rho(\s^*) + \vec{b} + \U\rho(\x)
    ,
\end{equation*}
where we have introduced $\varsigma(\s^*) \,{:=}\, \s^* \,{\oslash}\, \rho'(\s^*)$, with ${\oslash}$ representing the Hadamard (element-wise) division.
Furthermore, if $\rho$ is chosen such that $\varsigma$ is bijective (as is the case for most common choices of $\rho$), we may introduce a change of variables $\s_\varsigma^* := \varsigma(\s^*)$ and rearrange the DEQ to
\begin{equation*}
    \s_\varsigma^* = \W \rho(\varsigma^{-1}(\s_\varsigma^*)) + \vec{b} + \U\rho(\x)
    .
\end{equation*}
Under the substitution $\rho_\varsigma := \rho\circ \varsigma^{-1}$, we find
\begin{equation}
    \s_\varsigma^* = \W \rho_\varsigma(\s_\varsigma^*)) + \vec{b} + \U\rho_\varsigma(\varsigma(\x))
    \label{equation_chn_as_ham},
\end{equation}
which coincides exactly with \cref{explicit_naive_ham_deq}, the DEQ of a HAM, 
with non-linearity $\rho_\varsigma$ and input preprocessing using $\varsigma$.
\end{proof}
While \cref{equation_chn_as_ham} has direct access to the inverted function $\varsigma^{-1}$, a CHN does not. Instead, it uses its fixed point structure to approximate this inverse function during inference. Specifically, for a single state $s^*_i$, we have $s^*_i = \rho'(s^*_i) \cdot C_i$, where $C_i$ is a constant depending on the value of all other states, which we assume are kept fixed here. Depending on $\varsigma$, this fixed point equation may converge very slowly, therefore requiring many iterations. If an efficient implementation of $\varsigma^{-1}$ would be available, it would almost always be better to directly use \cref{equation_chn_as_ham} instead of \cref{implicit_naive_chn_deq}.
Conversely, certain activation functions may lead to a stable CHN, despite not resulting in a bijective $\varsigma$. Nevertheless, for some of these CHNs, an equivalent HAM may still be formulated.

\begin{example}
A CHN with $\rho \,{=}\, \text{ReLU}$, analytically extended such that $\rho'(0) \,{=}\, 1$, would result in a non-bijective $\varsigma$. Nonetheless, it can quickly be recognized as a HAM with the same non-linearity, but where the states are limited to be strictly non-negative. Moreover, the HAM is linear, and an analytic expression for the equilibrium may be found. By contrast, it is not clear how stable a naive implementation of this CHN would be, and using the HAM counterpart would likely be the better choice.
\end{example}

\section{Redundancy of synchronous updates}
\label{appendix_redundancy_of_synchronous_updates}
The state substitution in \cref{even_odd_deq_update_rule} reveals an interesting phenomenon arising in synchronously updated HAMs.
First, it is important to reiterate that, as mentioned in \cref{appendix_advantages_deq_framework}, iteratively applying \cref{eo_transform_chn_deq_update,eo_transform_ham_deq_update} is equivalent to minimizing the energy $E$ from \cref{energy_function_HAM} by solving the ODE of \cref{HAM_state_update_rule} using the forward Euler method with synchronous state updates and a time step size equal to 1.
As illustrated in \cref{figure_2deqs_in_hopfield}, this scenario corresponds exactly to simultaneously solving two DEQs of the form of \cref{even_odd_deq}, one at time step $n$ (solid), the other at $n+1$ (dashed).
In other words, synchronously updating the states corresponds to solving two internal DEQs with independent state dynamics.

This redundancy is not beneficial.
Below, we discuss two problems that arise in this case and describe how even-odd splitting avoids them.
As an alternative solution, one may also turn to a particular state initialization, which induces dynamics identical to even-odd splitting, albeit at a lower computational efficiency.

\paragraph{Problem \#1: Lack of convergence guarantees}\newline
\vspace{-11pt}
\setlength{\columnsep}{14pt}
\begin{wrapfigure}{r}{0.5\textwidth}
\vspace{-8pt}
\centering
\includeinkscape[width=0.95\linewidth, pretex=\footnotesize]{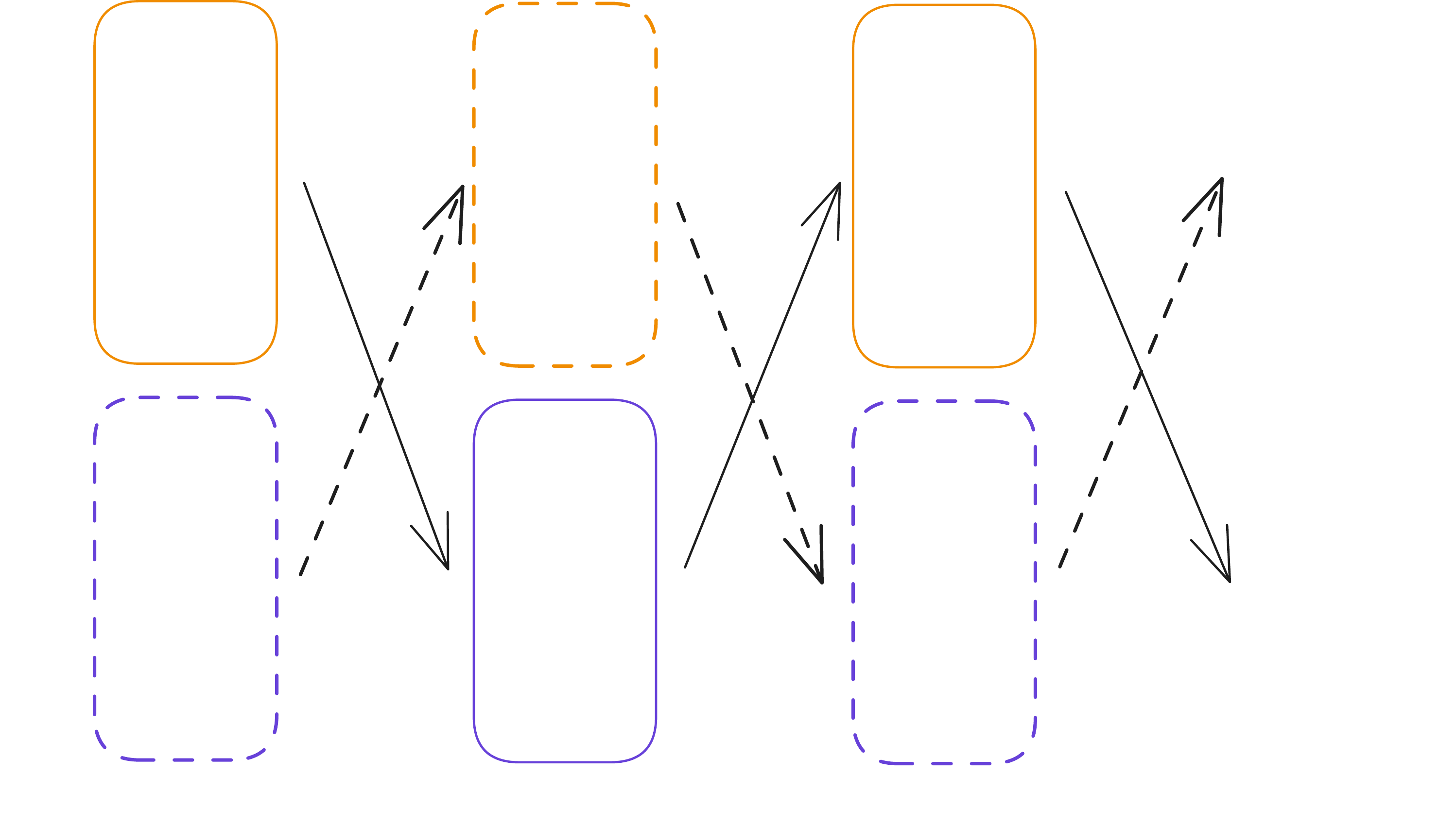}
\caption{A view of synchronous updates across time reveals two separate even-odd DEQs (solid \& dashed)}
\label{figure_2deqs_in_hopfield}
\vspace{-12pt}
\end{wrapfigure}
\cref{figure_2deqs_in_hopfield} nicely illustrates how state convergence under synchronous updates can only be guaranteed over two time steps (i.e., a length-2 limit cycle exists), as has long been known for Hopfield networks \parencite{koiran1994dynamics,wang1998_2cycle_convergence}.
Absolute convergence can only be achieved when both the solid and the dashed DEQ converge to the same equilibrium point. This may not always be the case, as it depends on both the input $\x$ and the specifics of the DEQ, such as state initialization, parameter values, and choice of non-linearity $\rho$.

\vspace{10pt}
However, this behavior can easily be guaranteed by simply iterating a single DEQ (e.g., the solid one in \cref{figure_2deqs_in_hopfield}) and defining the second DEQ as a time-shifted copy of the first one.
This is precisely what even-odd splitting does and what \cref{even_odd_deq} describes.

\paragraph{Problem \#2: Partial gradient flow}\newline
In practice, the total number of iterations is typically kept fixed and determined in advance.
In the case of \cref{figure_2deqs_in_hopfield}, this means that whichever DEQ contains the output prediction $\hat{\y}$ after the final iteration at time $T$, will always be the one receiving the gradient, which can then backpropagate through time, either explicitly or using memory-efficient DEQ methods.
The other DEQ will not receive any gradient information whatsoever, whereas its parameters, which are shared with the first DEQ, are still updated.

Since the gradients depend on only a single DEQ, the parameters will adapt to the initial behavior of that DEQ, which may not be optimal for the other one.
This might be problematic if the model is ever to run for a different number of iterations, with a different polarity.
In \cref{even_odd_deq}, this is never a problem.
As the DEQ internally advances two time steps at once, the final time $T$ will always be even. Additionally, it never explicitly models the other DEQ, whose behavior at initialization is therefore irrelevant.

\paragraph{Inducing even-odd splitting through state initialization}\newline
State initialization plays an important role in guaranteeing convergence.
To illustrate what might go wrong, let us assume that $n\,{=}\,0$ in \cref{figure_2deqs_in_hopfield} and that $\s^{0}_\text{even}$ and $\s^{0}_\text{odd}$ are initialized at zero, as is commonly done.
In the solid DEQ, $\s^{1}_\text{odd}$ receives information from both the input $\x$ and $\s^{0}_\text{even}$, and updates its states accordingly.
However, in the dashed DEQ, $\s^{1}_\text{even}$ only receives information from $\s^{0}_\text{odd}$, which is not input-dependent, and uses that to update its states.
This means that the dashed DEQ is not exactly a time-shifted version of the solid DEQ anymore: its initialization for $\s_\text{even}$ will be $\s^{1}_\text{even}$, which does not necessarily equal $\s^{0}_\text{even}$.

To avoid this discrepancy, we may design an initialization scheme such that $\s^{1}_\text{even} = \s^{0}_\text{even}$ by construction. Setting
\begin{equation*}
   \s^{0}_\text{even} = \vec{W_P^T}\rho(\s^{0}_\text{odd}) + \vec{b}_\text{even}
\end{equation*}
guarantees this equality, as can be seen from \cref{eo_transform_ham_deq_update}.
In fact, this initialization scheme directly induces even-odd splitting in synchronously updated HAMs. After all, we find that
\begin{equation*}
   \s^{0}_\text{even} = \s^{1}_\text{even} \Longrightarrow \s^{1}_\text{odd} = \s^{2}_\text{odd} \Longrightarrow \s^{2}_\text{even} = \s^{3}_\text{even} \Longrightarrow \ldots
   ,
\end{equation*}
which corresponds exactly to even-odd splitting, where the even/odd layers are alternately kept fixed for a single time step. Note, however, that synchronous updates still waste computations on these fixed values, making \cref{even_odd_deq_update_rule} the more sensible update rule to follow.

\section{Experimental setup}
\label{appendix_experimental_setup}
Below is an overview that should contain all information required to reproduce the results from \cref{section_results}.
The code is available at \url{https://github.com/cgoemaere/hopdeq}.

\begin{itemize}
    \item[] \textbf{Data}
    \item Dataset: EMNIST-MNIST \parencite{cohen2017emnist}. This is a drop-in replacement for the MNIST dataset \parencite{lecun1998mnist}, but with a known conversion process from the original NIST digits \parencite{grother1995nist}.
    \item Input preprocessing: rescaling pixel intensities from [0, 255] to [0, 1]
    \item Batch size: 64
    \item Epochs: 10
    \item No data augmentation
    \\
    \item[] \textbf{Model}
    \item Architecture (with a constant amount of hidden neurons)
    \begin{itemize}
        \item 3 layers: 784-1990-10
        \item 5 layers: 784-1280-510-200-10
        \item 7 layers: 784-1024-512-256-128-70-10
    \end{itemize} 
    \item Non-linearity $\rho$: sigmoid($4x-2$) (shifted sigmoid; same as \textcite{laborieuxScalingEquilibriumPropagation2020})
    \item State initialization: zero initialization, i.e., $\s^{n=0} = \mathbf{0}$
    \item Weight initialization: Xavier initialization \parencite{xavier_initialization} per layer (not on full $\W$, but on $\W_i$), as we want bidirectional operation between layers. The biases were initialized using a normal distribution with mean 0.0 and standard deviation 0.01.
    \item Forward iterations (chosen large enough to ensure state convergence during training): 40 (3 layers), 80 (5 layers), 120 (7 layers)
    \item DEQ solver: Anderson acceleration with windows size $m \,{=}\, 4$, Tikhonov regularization (constant: $10^{-10}$), and safe-guarding
    \item Damping (tuned to maintain stability during training)
    \begin{itemize}
        \item CHN: damping of 0.5, i.e., if the DEQ is $\s^* = f(\s^*)$, then we use $\s^{n+1} = 0.5\s^n + 0.5f(\s^n)$ as update rule. From an ODE perspective, this means that time moves half as fast (i.e., step size $h \,{=}\, 0.5$). To compensate for that, we multiply the provided number of forward and backward iterations with a factor 2, so that the same amount of ODE time is simulated.
        \item CHN-EO \& HAM: no damping, i.e., if the DEQ is $\s^* = f(\s^*)$, then we use $\s^{n+1} = f(\s^n)$ as update rule.
    \end{itemize}
    \item[] \textbf{Training}
    \item Loss function: Mean Square Error
    \item Backward method: Recurrent Backpropagation \parencite{almeida1987recurrentbackprop,pineda1987recurrentbackprop} with Picard iteration (always; Anderson acceleration was unstable here)
    \item Backward iterations: 8 (3 layers), 16 (5 layers), 24 (7 layers)
    \item Optimizer
    \begin{itemize}
        \item Type: Madam \parencite{bernstein2020madam} (chosen as a substitute for layerwise learning rates; Madam automatically scales weight updates according to $||\Delta W||/||W||$, as advised by \textcite{scellierEquilibriumPropagationBridging2017a})
        \item Learning rate: 0.01 (3 layers), 0.005 (5 layers \& 7 layers)
        \item Learning rate decay: linear decay to $1{/}10$\textsuperscript{th} of the initial learning rate mentioned above, over the course of the 10 epochs \parencite[inspired by][]{bernstein2020madam}
        \item Hyperparameters (see \href{https://github.com/jxbz/madam/blob/master/pytorch/optim/madam.py}{implementation}): p\_scale = 1024; g\_bound = 3
    \end{itemize}
    \item No gradient clipping, dropout or other commonly used training techniques
    \item GPU: 1x GTX-1080Ti
\end{itemize}

\clearpage
\section{Visual comparison of state dynamics in different configurations of Hopfield networks}
\label{appendix_visual_state_dynamics}
In \cref{figure_state_dynamics_3layers_CHN_HAM,figure_state_dynamics_5-7layers}, we provide a visual comparison of the state dynamics in the different models from \cref{section_results}. First, notice how the use of DEQ solvers helps guarantee convergence in samples that would otherwise not have converged. Additionally, even-odd splitting seems to boost convergence speed overall, by a factor close to two, as expected. We can see that the initial dynamics of the models differ from their regular regime, as the trajectories of all samples start out similarly and only diverge after a few iterations. As for the low density region in the models using DEQ solvers (most noticeable in the bottom right subplots), we hypothesize that this is due to the solver occasionally finding the exact fixed point solution, bringing the relative residual to zero.
\vspace{10pt}

\begin{figure}[!ht]
    \centering
    \includegraphics[width=.49\linewidth]{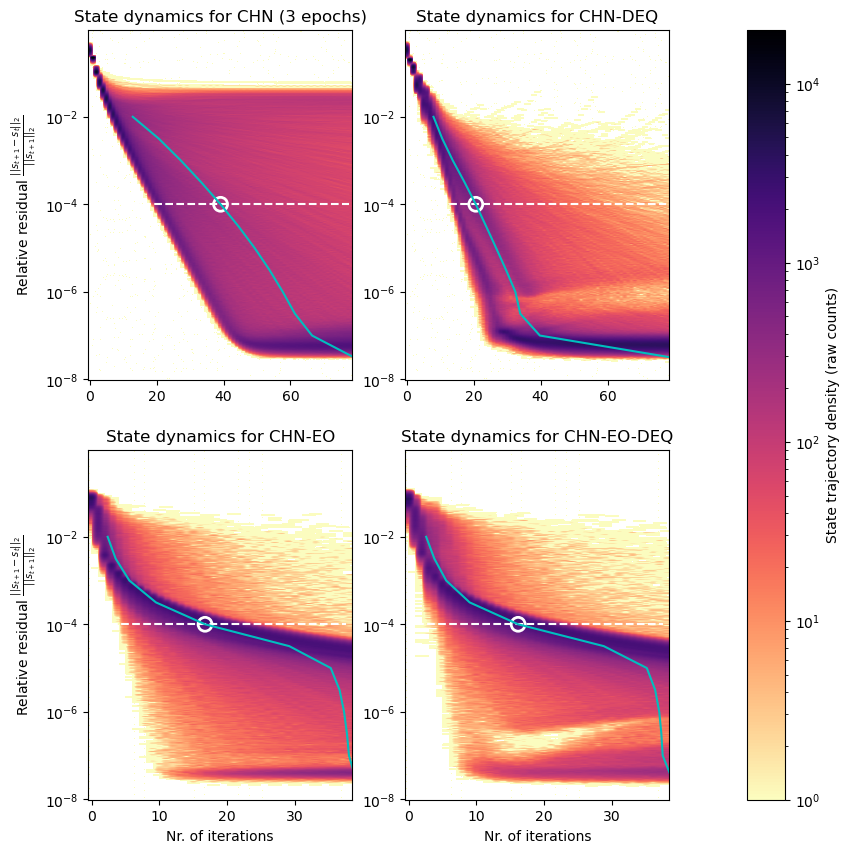}
    \hfill
    \includegraphics[width=.49\linewidth]{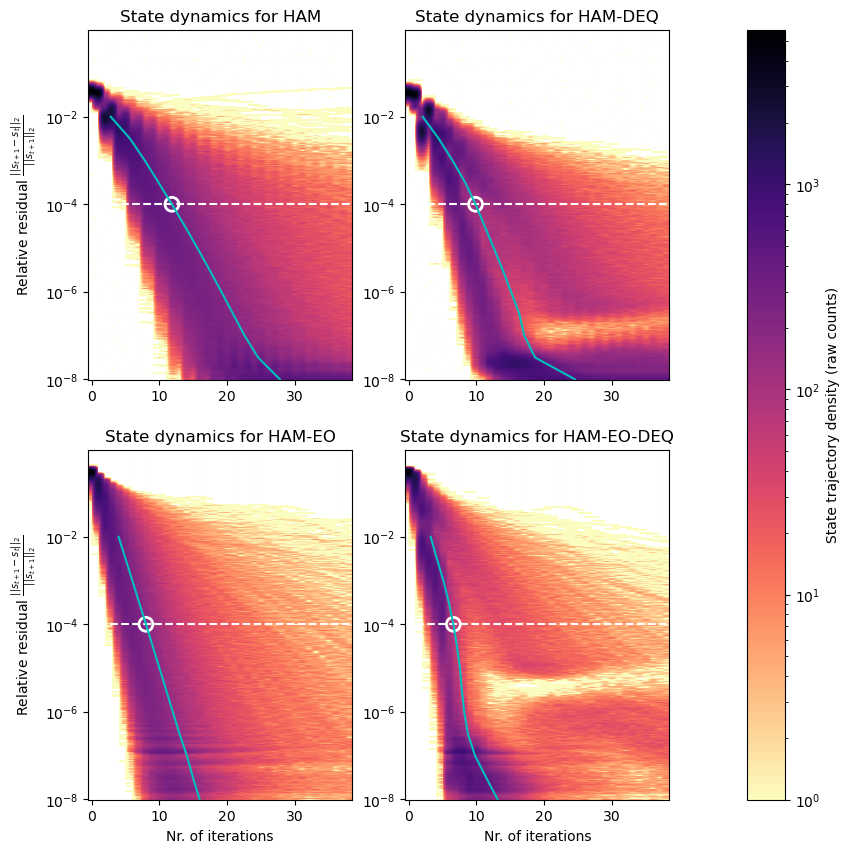}
    \caption{Density heatmap of the state trajectories for a 3-layer CHN (left) and HAM (right), and the impact of using DEQ solvers (`DEQ') and even-odd splitting (`EO'). The horizontal axis represents the number of iterations of the DEQ. The vertical axis represents the relative residual, which is used to determine the state convergence (the lower, the more converged). The limit of $10^{-4}$ as chosen criterion for convergence is indicated with a white dashed line. For every setting, we show the cumulative results of 5 different seeds, run on the entire MNIST test set. In cyan, we show the mean number of iterations corresponding to a given convergence criterion. The white circular marker at the limit of $10^{-4}$ corresponds to the value reported in \cref{table_results}.}
    \label{figure_state_dynamics_3layers_CHN_HAM}
\end{figure}

\begin{figure}[!ht]
    \centering
    \includegraphics[width=.49\linewidth]{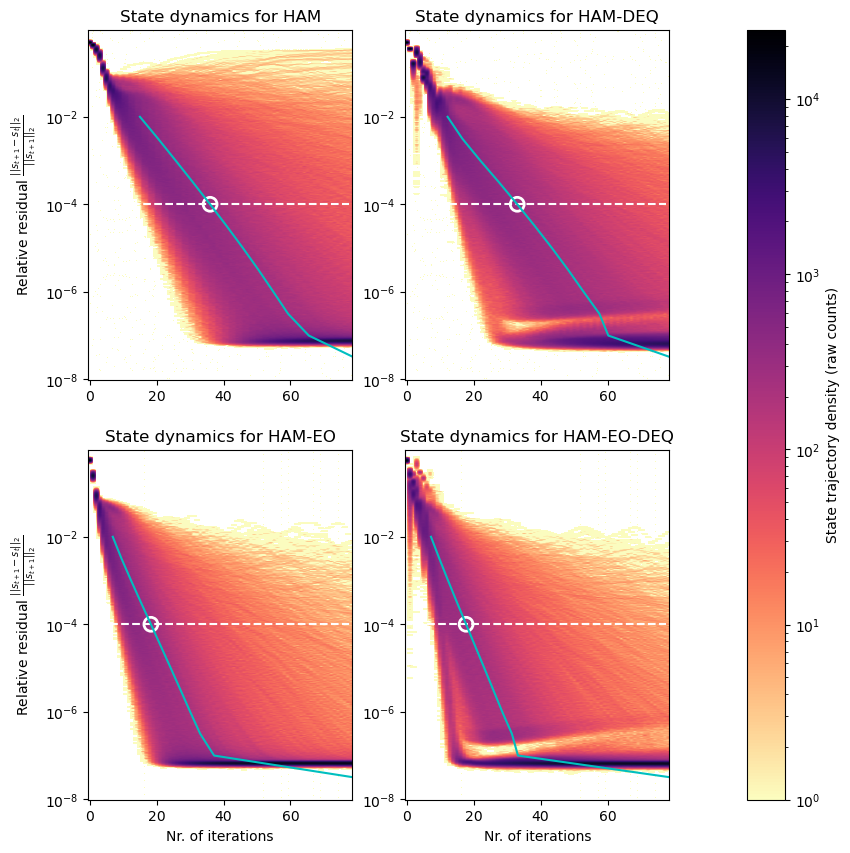}
    \includegraphics[width=.49\linewidth]{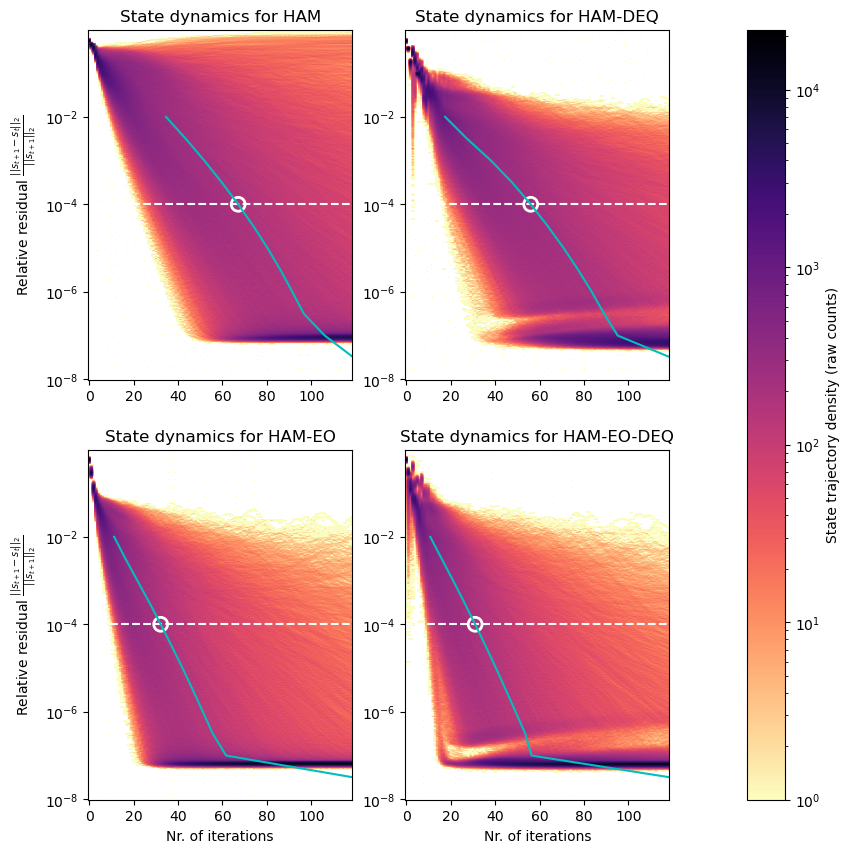}
    \caption{Density heatmap of the state trajectories for a 5-layer (left) and 7-layer HAM (right), and the impact of using DEQ solvers (`DEQ') and even-odd splitting (`EO'). The horizontal axis represents the number of iterations of the DEQ. The vertical axis represents the relative residual, which is used to determine the state convergence (the lower, the more converged). The limit of $10^{-4}$ as chosen criterion for convergence is indicated with a white dashed line. For every setting, we show the cumulative results of 5 different seeds, run on the entire MNIST test set. In cyan, we show the mean number of iterations corresponding to a given convergence criterion. The white circular marker at the limit of $10^{-4}$ corresponds to the value reported in \cref{table_results}.}
    \label{figure_state_dynamics_5-7layers}
\end{figure}

\end{document}